\documentclass[conference]{IEEEtran}
\IEEEoverridecommandlockouts
\usepackage{indentfirst}
\usepackage[dvips]{graphicx}
\usepackage{amsfonts}
\usepackage{multirow}

\usepackage{amsmath,amsthm,amssymb}
\usepackage{color,changepage}
\usepackage{algorithm}
\usepackage{algorithmic}
\usepackage{bbm}
\usepackage{verbatim}
\usepackage{makecell}
\usepackage{subfigure}
\usepackage{mathrsfs}
\usepackage{arydshln}
\usepackage{cases}
\usepackage{threeparttable}
\usepackage{extarrows}
\usepackage{array}
\usepackage{bm}
\usepackage{pifont}
\usepackage[T1]{fontenc}
\usepackage{threeparttable}
\usepackage{booktabs}
\usepackage{xspace}
\usepackage{mdwlist}
\usepackage{hyperref}
\usepackage{cite}
\allowdisplaybreaks[4]

\definecolor{newcolor}{rgb}{0.5,0,1}

\newcommand{\scheme}{SecEA\xspace}

\theoremstyle{remark}

\newcolumntype{I}{!{\vrule width 2pt}}

\newtheorem{theorem}{Theorem}

\newtheorem{lemma}{Lemma}

\theoremstyle{remark}
\newtheorem{Remark}{Remark}

\DeclareMathAlphabet{\mathpzc}{OT1}{pzc}{m}{it}

\definecolor{newcolor}{rgb}{0.5,0,1}

\def\BibTeX{{\rm B\kern-.05em{\sc i\kern-.025em b}\kern-.08em
    T\kern-.1667em\lower.7ex\hbox{E}\kern-.125emX}}

\begin{document}

\title{Secure Embedding Aggregation for Federated Representation
Learning}

\author{\IEEEauthorblockN{}}
\author{\IEEEauthorblockN{Jiaxiang Tang$^{1,2}$, Jinbao Zhu$^{2}$, Songze Li$^{1,2}$, Lichao Sun$^{3}$}
\IEEEauthorblockA{$^{1}$The Hong Kong University of Science and Technology
\\ $^{2}$The Hong Kong University of Science and Technology (Guangzhou)\\
$^{3}$Lehigh University\\
E-mails: jtangbe@connect.ust.hk, jbzhu@ust.hk, songzeli@ust.hk, lis221@lehigh.edu}
}

\maketitle

\begin{abstract}
We consider a federated representation learning framework, where with the assistance of a central server, a group of $N$ distributed clients train collaboratively over their private data, for the representations (or embeddings) of a set of entities (e.g., users in a social network). Under this framework, for the key step of aggregating local embeddings trained privately at the clients, we develop a secure embedding aggregation protocol named \scheme, which leverages all potential aggregation opportunities among all the clients, while providing  privacy guarantees  for the set of local entities and corresponding embeddings \emph{simultaneously} at each client, against a curious server and up to $T < N/2$ colluding clients. 
\end{abstract}


\section{Introduction}

Federated learning (FL)~\cite{mcmahan2017communication,FL3} is an emerging privacy-preserving collaborative learning paradigm. With the help of a central server, a group of distributed clients collaboratively train a high-performance global model without revealing their private data. 
Recently, FL framework is applied to federated representation learning (FRL) \cite{wu2021fedgnn,chen2020fede,fedgraphnn},
in which the goal is to train good representations (or embeddings), for each entity (e.g., users in a social network), over the private data distributed on the clients. A typical training round of an FRL protocol consists of the following steps: (i) each client trains the local embedding for each of its entities using its private data; (ii) all the clients send their trained local embeddings to the server; (iii) the server aggregates the local embeddings from different clients with the same entity into a global embedding; and (iv) the server sends the global embeddings back to the clients for the training of the next round. 

Aggregating embeddings of the same entities over all clients helps to 
enhance the embedding quality and the learning performance, for a wide range of representation learning tasks (e.g., recommendation system \cite{wu2021fedgnn,chai2020secure}, social network mining \cite{kolluri2021private}, and knowledge graph \cite{chen2020fede}).
To this end, FRL first needs to align the local entities of the clients, and then exchanges embeddings to perform aggregation for each entity.
However,
during the embedding aggregation process, the curious server and clients can potentially infer the local entities and their embeddings of the victim clients, which would lead to leakage of the victim clients' local datasets. 
To protect the privacy of clients' local entities, the current state-of-the-art approach is for the FRL system to first privately agree on the set of entities that are common to all clients, using private set intersection (PSI) primitives (see, e.g.,~\cite{psi:5,psi:2}). Next, for each common entity existing on all clients, the clients securely aggregate their local embeddings, using secure aggregation protocols that mask the embeddings with random noises (see, e.g.,~\cite{SecureFL:1,SecureFL:3,SecureFL:5,SecureFL:4,jahani2023swiftagg+,jahani2022swiftagg,shao2022dres}). However, with the idea of aggregating embeddings of entities common to \emph{all} clients, PSI-based approaches suffer from  1) privacy leakage: the existence of the common entities at all other clients is known at each client; and 2) performance degradation: aggregation opportunities among \emph{subsets} of clients who share common entities are not leveraged.


In this paper, we propose a novel secure embedding aggregation protocol, named \scheme, which \emph{simultaneously} provides entity privacy and embedding privacy for FRL, and overcomes both shortcomings of PSI-based approaches. To address the privacy challenges in the secure embedding aggregation problem, \scheme utilizes techniques from Lagrange multi-secret sharing~\cite{beimel2011secret,yu2019lagrange}, private information retrieval (PIR) \cite{chor1995private,sun2017capacity,banawan2018capacity,zhu2019new,ulukus2022private}, and Paillier's homomorphic encryption \texttt{PHE} \cite{paillier}.
A PIR protocol allows a client to retrieve an interested item (or embedding) from a set of databases without revealing to the databases which item is being retrieved, thus it can be used as a building block to protect entity privacy while leveraging the aggregation opportunities among subsets of clients. 


Specifically, in our \scheme protocol, the FRL system first performs a one-time private entity union operation, such that each client learns the collection of the entities existing on all clients. 
In each global training round, to compute the average embedding aggregation, each client first expands each of its local embedding vectors to include an indicator variable that indicates the existence of an entity.  
Then each client secret shares the expanded embedding vectors with the other clients, such that secret shares of all aggregated embeddings can be obtained at all clients, 
which is compatible with the symmetric PIR problem from secure MDS-coded storage system \cite{zhu2021multi,zhu2022symmetric}.
To privately obtain the average embedding for a local entity, each client sends a \emph{coded} query, 
to every other client, who returns a \texttt{PHE}-encrypted response through the server. 
To further protect entity privacy, 
the server adds a random noise on each response, such that the client can recover the average embedding of the intended entity, without knowing how many other clients also having this entity.

A related problem of private federated submodel learning (PFSL) has recently been studied in \cite{FSL:6,FSL:3,vithana2022private,FSL:5,FSL:2}, where a client would like to update one of many submodels, while keeping the submodel index and the updates private. While both utilize techniques from solving PIR problems to protect the privacy of model index (entity) and model update (embedding), our proposed \scheme protocol focuses more on privately aligning local entities of clients and aggregating \emph{average} embeddings of the same entities across all clients, which are not required in the PFSL problem.

\noindent {\bf Notation.} For two integers $m \leq n \in \mathbb{Z}$, we define $[m] \triangleq \{1,\ldots,m\}$ and $[m,n]\triangleq  \{m,m+1,\ldots,n\}$. 

\section{Problem Formulation}\label{problem statement}
Consider a representation learning task with a dataset $\mathcal{D} = (\mathcal{E},\mathcal{X})$, where $\mathcal{X}$ is a collection of data points (e.g., user information in a social network), and $\mathcal{E}$ is the set of entities of the data points in ${\mathcal{ X}}$ (e.g., IDs of the users). 
The goal is to train a collection of embedding vectors
$\mathcal{H}=\{\mathbf{h}_{e}: e\in\mathcal{E}\}$, by minimizing some loss function $\mathcal{L}(\mathcal{ D}, {\mathcal {H}})$, where $\mathbf{h}_e \in \mathbb{F}^d$ denotes the corresponding embedding vector of length~$d$ for each entity $e \in {\mathcal {E}}$, and $\mathbb{F}$ is some sufficiently large finite field.\footnote{As secure computation protocols are built upon cryptographic primitives that carry out operations over finite fields, we consider each element of an embedding vector to be from a finite field $\mathbb{F}$.}

For an FRL system consisting of a central server and $N$ clients, each client executes the above task on its local data. Specifically, for each $n \in [N]$, client $n$ has a local dataset $\mathcal{D}_n=(\mathcal{E}_n, \mathcal{X}_n)$ consisting of a set of local entities $\mathcal{E}_n$ and a set of local data points $\mathcal{X}_n$.
The entity sets across clients may overlap arbitrarily, i.e., $\mathcal{E}_n\cap\mathcal{E}_{v}\ne\varnothing$ for $n \ne v$.
In each round $t$ of FRL, with the knowledge of global embedding $\mathbf{h}_{e}^{(t-1)}$ for each $e\in\mathcal{ E}_n$, client $n$ trains its local embeddings over $\mathcal{ D}_n$:
\begin{equation}
\label{eq:local update}
\{\mathbf{h}_{n,e}^{(t)}:e \in \mathcal{ E}_n\}=\arg\min_{\mathcal{H}_n}\mathcal{L}(\mathcal{D}_n,\mathcal{H}_n),
\end{equation}
where $\mathcal{H}_n=\{\mathbf{h}_{e}^{(t-1)},e\in\mathcal{E}_n\}$, and $\mathbf{h}_{n,e}^{(t)}$ denotes the updated local embedding of entity $e$ at client $n$ in round $t$. Having updated their local embeddings, the $N$ clients communicate these embeddings with each other, through the central server,
such that for each $n\in [N]$ and each $e \in \mathcal{ E}_n$, client $n$ obtains an updated global embedding of entity $e$, which is computed by 
averaging the local embeddings of $e$ from all the clients who have $e$ locally. More precisely, 
for each entity $e$, the global embedding $\mathbf{h}_{e}^{(t)}$ for the next round is computed as
\begin{eqnarray}\label{learning:aggregation}
\mathbf{h}_{e}^{(t)}=\frac{\sum_{v\in[N]}\mathbbm{1}(e\in\mathcal{E}_{v})\cdot\mathbf{h}_{v,e}^{(t)}}{\sum_{v\in[N]}\mathbbm{1}(e\in\mathcal{E}_{v})}, \label{eq:task}
\end{eqnarray}
where $\mathbbm{1}(x)$ is the indicator function that returns $1$ when $x$ is true and $0$ otherwise. 
The global embeddings of the entities are updated iteratively until convergence.

\noindent {\bf Threat Model.} We consider \emph{honest-but-curious} adversaries, which is the common model adopted to study privacy vulnerabilities in FL systems \cite{leakage:3,SecureFL:4,passive2}. Specifically, corrupted parties (clients and the server) will faithfully follow the learning protocol but will try to infer a client's private information including its entity set and data samples. The adversary can corrupt the server or multiple clients, but not the server and clients simultaneously, i.e., the server does not collude with clients to infer private information of other clients. Moreover, we consider a static adversary, who corrupts the same parties in each round.

\noindent {\bf Security Goals.}
In the above-described FRL framework, we focus on two types of data privacy under the threat model: entity privacy and embedding privacy. More concretely, given a security parameter $T < N$, an embedding aggregation protocol is considered \emph{$T$-private} if the following requirements are simultaneously satisfied:

\begin{itemize}
  \item \emph{Entity-privacy.} The entity set of any individual client must
be kept private from the server and the remaining clients,
even if any up to $T$ clients collude to share information
with each other. In other words, the server or any subset
of $T$ colluding clients learn nothing about which entities
are owned by each of the other clients.\footnote{Note that when an entity $e$ is only owned by corrupted clients, entity privacy leakage is inevitable, i.e., the colluding clients owning entity $e$ know that other clients do not have it by comparing the global embedding of $e$ with their local embedding average. WLOG, we discard this corner case throughout the paper.}

  \item \emph{Embedding-privacy.} Any subset of up to $T$ colluding clients learn nothing about the local embeddings of the other clients, beyond the embedding aggregations of the colluding clients' local entities. Besides, the corrupted server learns nothing about the local embedding of all the clients.
  
\end{itemize}



The goal of this paper is to design a provably secure embedding aggregation protocol, for general FRL tasks.


\section{\scheme Protocol}
Our proposed \scheme protocol consists of two main components: \emph{private embedding sharing}, and \emph{private embedding aggregation retrieval}. 
We present the general protocol and illustrate its core ideas via a simple example.


\subsection{Private Embedding Sharing}


Before the training starts, all clients execute a one-time private entity set union protocol proposed in \cite{PSU}, for each client and the server to privately obtain the union of the entity sets from all clients, without knowing the entity set of any individual client.\footnote{Here all communications between clients through the server are encrypted using one-time-pad private key encryption, so no information about clients' entities is leaked to the server.} 
Let the global entity set $\mathcal{ E} = \bigcup_{n\in[N]}\mathcal{E}_n = \{e_1,\ldots,e_M\}$, where $M$ is the total number of distinct entities across all clients.
After the initial private entity union operation, the set $\mathcal{ E}$ is known to all clients.

During each training round, each client $n$ computes a set of local embeddings $\{\mathbf{h}_{n,e_m}:e_m \in \mathcal{ E}_n\}$ as in \eqref{eq:local update}, where we omit the round index $t$ for brevity.
To proceed, we expand and redefine the local embedding ${\mathbf{h}}_{n,e_m}$ of each entity $e_m$ at each client $n$ as a vector $\Tilde{\mathbf{h}}_{n,e_m}$ with dimension $d+1$, given by
\begin{eqnarray}\label{12345}
\Tilde{\mathbf{h}}_{n,e_m}=
\left\{\begin{array}{@{}ll}
(\mathbf{h}_{n,e_m},1),  & \text{if}~e_m\in\mathcal{E}_n  \\
\mathbf{0},&\text{otherwise}
\end{array}\right..
\end{eqnarray}

To perform secure embedding aggregation, each client secret shares its expanded embeddings with other clients using Lagrange encdoing~\cite{yu2019lagrange}. For reducing the communication cost among the clients, given the security parameter $T$, we select a partitioning parameter $K$, such that,
\begin{eqnarray}\label{partition number}
K=\left\lfloor\frac{N+1}{2}\right\rfloor-T.
\end{eqnarray}
Then, for each $m\in[M]$, client $n$ evenly partitions 
$\Tilde{\mathbf{h}}_{n,e_m}$ into $K$ sub-vectors of dimension $\frac{d+1}{K}$, i.e., 
\begin{IEEEeqnarray}{c}\label{partion:embeddsd}
    \Tilde{\mathbf{h}}_{n,e_m}=\big(
    \Tilde{\mathbf{h}}_{n,e_m}^{1},\ldots,\Tilde{\mathbf{h}}_{n,e_m}^{K}
\big).
\end{IEEEeqnarray}

Let $\{\beta_{k},\alpha_{n}:k\in[K+T],n\in[N]\}$ be  $K+T+N$ pairwise distinct parameters known to the 
system from $\mathbb{F}$. 
Each client $n\in [N]$, for each $m\in[M]$, samples independently and uniformly over $\mathbb{F}^{\frac{d+1}{K}}$, $T$ random noises $\mathbf{z}^{K+1}_{n,e_m},\mathbf{z}^{K+2}_{n,e_m},\ldots,\mathbf{z}^{K+T}_{n,e_m}$,  
and then constructs a polynomial $\varphi_{n,e_m}(x)$ of degree at most $K+T-1$ such that
\begin{eqnarray}\label{encoding:data}
\varphi_{n,e_m}(\beta_{k})=\left\{
\begin{array}{@{}ll}
\Tilde{\mathbf{h}}^{k}_{n,e_m},&\forall\, k\in[K]\\
\mathbf{z}^{k}_{n,e_m},&\forall\, k\in[K+1:K+T]
\end{array}\right..
\end{eqnarray}

Then, for each $v \in [N]$, client $n$ shares the evaluation of $\varphi_{n,e_m}(x)$ at point $x=\alpha_v$ with client $v$.
The secret shares sent by client $n$ to client $v$ across all $m\in[M]$ are given by
\begin{IEEEeqnarray}{rCl}\label{stored data}
\mathbf{y}_{n,v}&=&\Big(\varphi_{n,e_1}(\alpha_v),\ldots,\varphi_{n,e_M}(\alpha_v)\Big).
\end{IEEEeqnarray}
Notably, since the sharing messages are sent through the relay of the central server, to protect information on entities and embeddings from leaking to the server, each client sends a masked version of its sharing messages using one-time-pad encryption.
Specifically, each pair of clients $n,v$ agree on a private seed $a_{n,v}$ unknown to the server using a Diffie-Hellman type key exchange protocol~\cite{diffie1976new}. When client $n$ wishes to send $\mathbf{y}_{n,v}$ to client $v$, the communication takes place in the following steps:
1) client $n$ uploads $\tilde{\mathbf{y}}_{n,v}=\mathbf{y}_{n,v}+\textup{PRG}(a_{n,v})$ to the server using a pseudorandom generator (PRG) generated one-time-pad, and then the server forwards the received $\tilde{\mathbf{y}}_{n,v}$ to client $v$.
2) client $v$ decrypts the desired data $\mathbf{y}_{n,v}$ by performing $\tilde{\mathbf{y}}_{n,v}-\textup{PRG}(a_{n,v})=\mathbf{y}_{n,v}$. 

After receiving and decrypting sharing messages from all the clients, client $v$ aggregates them to obtain the following.
\begin{IEEEeqnarray}{rCccl}
\mathbf{y}_{v}\!&\triangleq&\!\!\sum \limits_{n\in[N]}\!\mathbf{y}_{n,v}&=&
\bigg(\sum\limits_{n\in[N]}\!\varphi_{n,e_1}(\alpha_v),\!\ldots\!,\!\!\sum\limits_{n\in[N]}\!\varphi_{n,e_M}(\alpha_v)\!\bigg)\!. \IEEEeqnarraynumspace\label{local stored222}
\end{IEEEeqnarray}

\subsection{Private  Embedding Aggregation Retrieval}
After the private embedding sharing, each client locally obtains secret shares of global embedding aggregations of all $M$ entities. To privately retrieve embedding aggregations for entities in $\mathcal{E}_n$ without revealing local entities, each client $n$ sends some queries to each other clients $v$ in a private manner. Client $v$ responds with some answers following the instructions of the received queries. Finally, client $n$ reconstructs the desired embedding aggregations from the answers. 


For an intended entity $e\in\mathcal{E}_n$ owned by client $n$, the client independently and uniformly generates $MT$ random variables $\{z_{n,e}^{m,K+1},\ldots,$ $z_{n,e}^{m,K+T}\}_{m\in[M]}$ from $\mathbb{F}$. Then, for each $m\in[M]$, the client constructs a query polynomial  $\rho_{n,e}^{m}(x)$ of degree $K+T-1$ 
such that  
\begin{IEEEeqnarray}{rCll}\label{symmetric:22}
\rho_{n,e}^{m}(\beta_{k})&=&\left\{
\begin{array}{@{}ll}
1, &\mathrm{if}\,\, e_m=e\\
0, & \mathrm{otherwise}
\end{array}
\right., ~\forall\, k\in[K],\label{query} \\
\rho_{n,e}^{m}(\beta_{k})&=&z_{n,e}^{m,k},
~\;\forall\, k\in[K+1:K+T]. 
\end{IEEEeqnarray}

Next, for each $v \in [N]$, client $n$ evaluates the $M$ query polynomials $\{\rho_{n,e}^{m}(x):m\in[M]\}$ at $x=\alpha_{v}$, and sends them to client $v$ using one-time-pad encryption with another pair-wise private seed $a_{n,v}'$.
We denote the query sent from client $n$ to $v$, for retrieving the aggregation of entity $e\in\mathcal{E}_n$, as
\begin{eqnarray}\label{query:1}
\mathbf{q}_{n,v,e}=\left(\rho_{n,e}^{1}(\alpha_v),\ldots,\rho_{n,e}^{M}(\alpha_v)\right).
\end{eqnarray}
Then, the encrypted query, $\Tilde{\mathbf{q}}_{n,v,e}=\mathbf{q}_{n,v,e}+\textup{PRG}(a_{n,v}')$, is sent to the client $v$ through the server. After decrypting the query received from client $n$, client $v$ takes the inner products of the query vector $\mathbf{q}_{n,v,e}$ and its locally stored data vector $\mathbf{y}_{v}$ \eqref{local stored222}, generating the response $A_{v,n,e}=\langle\mathbf{q}_{n,v,e}, \mathbf{y}_{v}\rangle$.

To protect entity privacy and complete desired embedding averaging, we apply Paillier's homomorphic encryption \texttt{PHE} \cite{paillier} on the response, where the \texttt{PHE} encryption algorithm \texttt{PHE.Enc} and decryption algorithm \texttt{PHE.Dec} have the additive property such that $\texttt{PHE.Enc}(m_1, pk)\cdot\texttt{PHE.Enc}(m_2, pk)\!=\!\texttt{PHE.Enc}(m_1+m_2, pk)$ and $\texttt{PHE.Dec}(\texttt{PHE.Enc}(m_1+m_2,pk),sk)=m_1+m_2$ for any messages $m_1,m_2$ and  a public key/secret key pair $(pk,sk)$. Specifically, client $n$ generates a pair of public and private keys $(pk_n, sk_n)$, where $pk_n$ is public to the server and other clients. Then the client $v$ sends a \texttt{PHE}-encrypted version $\Tilde{A}_{v,n,e}=\texttt{PHE.Enc}(A_{v,n,e},pk_n)$ of the response $A_{v,n,e}$ to server.

To prevent client $n$ from inferring any additional information about the embeddings of the entities that are not in $\mathcal{ E}_n$, the server also generates $\widetilde{K}=K+2T-1$ random noises $\{\mathbf{z}_{n,e}^{k}:k\in[\widetilde{K}]\}$ independently and uniformly from $\mathbb{F}^{\frac{d+1}{K}}$. Define a noise polynomial $\psi_{n,e}(x)$ of degree $2(K+T-1)$ such that
\begin{IEEEeqnarray}{c}
\psi_{n,e}(\beta_{k})=0,\quad 
\psi_{n,e}(\alpha_{k})=\mathbf{z}_{n,e}^{k},\quad k\in[\widetilde{K}]. 
\label{symmetric:1345}
\end{IEEEeqnarray}
Furthermore, to prevent client $n$ from learning the number of clients that owns the entity $e$, the server locally chooses a random noise $r_{n,e}$ from $\mathbb{F}$. Then the server adds the noise $r_{n,e}$ and the evaluation of $\psi_{n,e}(x)$ at $x=\alpha_v$ on the encrypted response $\Tilde{A}_{v,n,e}$, and generates $\Tilde{Y}_{v,n,e}$ for client $n$:
\begin{align}
 \Tilde{Y}_{v,n,e}&=(\Tilde{A}_{v,n,e})^{r_{n,e}}\cdot\texttt{PHE.Enc}(\psi_{n,e}(\alpha_{v}),pk_n) \label{answers:add} \\ &=\texttt{PHE.Enc}\Big(r_{n,e} A_{v,n,e}+\psi_{n,e}(\alpha_v),pk_n\Big),\notag
\end{align}
where the operations on vectors are performed element-wise,
and the last equality is due to the additive property of \texttt{PHE}. 

Client $n$ decrypts $\Tilde{Y}_{v,n,e}$ via its private key $sk_n$ and obtains $Y_{v,n,e}=r_{n,e} \cdot A_{v,n,e}+\psi_{n,e}(\alpha_v)$.
It is easy to check that the response $Y_{v,n,e}$ is the evaluation of the following response polynomial $Y_{n,e}(x)$ with degree $2(K+T-1)$ at point $x=\alpha_{v}$.
\begin{eqnarray}
Y_{n,e}(x)=r_{n,e} \sum\limits_{m=1}^{M}\rho_{n,e}^{m}(x)\cdot \sum\limits_{v'\in[N]}\varphi_{v',e_m}(x)+\psi_{n,e}(x). \label{answer:polynomial}
\end{eqnarray}


Recall that $\{\alpha_v\}_{v\in[N]}$ are distinct elements from $\mathbb{F}$.
Client $n$ can recover the polynomial $Y_{n,e}(x)$ from the received $N$ responses $(Y_{1,n,e},\ldots,Y_{N,n,e})=(Y_{n,e}(\alpha_1),\ldots,Y_{n,e}(\alpha_N))$ via polynomial interpolation as $2(K+T-1)<N$ by \eqref{partition number}.
For each $k\in[K]$, client $n$ evaluates $Y_{n,e}(x)$ at $x=\beta_{k}$ to obtain
\begin{IEEEeqnarray}{rCl}
Y_{n,e}(\beta_{k})&=&r_{n,e} \sum\limits_{m=1}^{M}\rho_{n,e}^{m}(\beta_k)\cdot \sum\limits_{v\in[N]}\varphi_{v,e_m}(\beta_k)+\psi_{n,e}(\beta_k)\notag \\
&\overset{(a)}{=}&r_{n,e}\sum\limits_{v\in[N]}\varphi_{v,e}(\beta_{k}) \overset{(b)}{=}r_{n,e}\sum\limits_{v\in[N]}\Tilde{\mathbf{h}}_{v,e}^{k},\label{evaluating:1}
\end{IEEEeqnarray}
where $(a)$ is due to \eqref{query} and \eqref{symmetric:1345}, and $(b)$ follows by \eqref{encoding:data}. Next, by \eqref{12345} and \eqref{partion:embeddsd}, we obtain
\begin{IEEEeqnarray}{l}
\Big(r_{n,e}\sum\limits_{v\in[N]}\Tilde{\mathbf{h}}_{v,e}^{1},\ldots,r_{n,e}\sum\limits_{v\in[N]}\Tilde{\mathbf{h}}_{v,e}^{K}\Big)= \notag\\
\quad\quad\Big(r_{n,e}\sum_{v\in[N]}\mathbbm{1}(e\in\mathcal{E}_{v})\cdot\mathbf{h}_{v,e},r_{n,e}\sum_{v\in[N]}\mathbbm{1}(e\in\mathcal{E}_{v})\Big).\IEEEeqnarraynumspace\label{rnoise}
\end{IEEEeqnarray}
Thus, client $n$ can correctly recover the global embedding of entity $e$, i.e., $\frac{\sum_{v\in[N]}\mathbbm{1}(e\in\mathcal{E}_{v})\cdot\mathbf{h}_{v,e}}{\sum_{v\in[N]}\mathbbm{1}(e\in\mathcal{E}_{v})}$, as in \eqref{learning:aggregation}.
Finally, each client $n\in[N]$ repeats the above process for each entity $e \in \mathcal{ E}_n$.


\subsection{Illustrative Example}


We illustrate the key ideas behind the proposed \scheme protocol through a simple example with $N=3$ and $K=T=1$.
Assume that the entire system contains $M=2$ entities, and their distributions onto the $3$ clients are $\mathcal{E}_1=\{e_1\},\mathcal{E}_2=\{e_2\}$ and $\mathcal{E}_3=\{e_1\}$, respectively.
The proposed \scheme protocol operates in two phases as follows. 


\noindent\textbf{Private Embedding Sharing.}
The system executes the private entity union protocol, for the server and all $3$ clients to agree on the global set of entities 
$\mathcal{E}=\{e_1,e_2\}$. 
In each global, after local updating, the expanding embeddings are given by
\begin{alignat*}{3}
&\Tilde{\mathbf{h}}_{1,e_1}=(\mathbf{h}_{1,e_1},1),&\;&\Tilde{\mathbf{h}}_{2,e_1}=(\mathbf{0},0),&\;&\Tilde{\mathbf{h}}_{3,e_1}=(\mathbf{h}_{3,e_1},1);\\ &\Tilde{\mathbf{h}}_{1,e_2}=(\mathbf{0},0),
 &\;&\Tilde{\mathbf{h}}_{2,e_2}=(\mathbf{h}_{2,e_2},1), &\;&\Tilde{\mathbf{h}}_{3,e_2}=(\mathbf{0},0).
\end{alignat*}

We select $\{\alpha_1,\alpha_2,\alpha_3\}=\{3,4,5\}$ and $\{\beta_1,\beta_2\}=\{1,2\}$. Each client $n\in [3]$ creates the following masked shares ${\bf y}_{n,v}^e$ for each $e=e_1,e_2$ using the noises $\mathbf{z}_{1,e},\mathbf{z}_{2,e},\mathbf{z}_{3,e}$ sampled uniformly at random, and shares it with each client $v\! \in\! [3]$.
\begin{alignat*}{2}
&{\bf y}_{1,1}^{e_1}= -\Tilde{\mathbf{h}}_{1,e_1}+2\mathbf{z}_{1,e_1}, \quad &&{\bf y}_{1,1}^{e_2}= -\Tilde{\mathbf{h}}_{1,e_2}+2\mathbf{z}_{1,e_2};\\
&{\bf y}_{1,2}^{e_1}= -2\Tilde{\mathbf{h}}_{1,e_1}+3\mathbf{z}_{1,e_1}, \quad &&{\bf y}_{1,2}^{e_2}= -2\Tilde{\mathbf{h}}_{1,e_2}+3\mathbf{z}_{1,e_2};\\
&{\bf y}_{1,3}^{e_1}= -3\Tilde{\mathbf{h}}_{1,e_1}+4\mathbf{z}_{1,e_1}, \quad &&{\bf y}_{1,3}^{e_2}= -3\Tilde{\mathbf{h}}_{1,e_2}+4\mathbf{z}_{1,e_2};\\
&{\bf y}_{2,1}^{e_1}= -\Tilde{\mathbf{h}}_{2,e_1}+2\mathbf{z}_{2,e_1}, \quad &&{\bf y}_{2,1}^{e_2}= -\Tilde{\mathbf{h}}_{2,e_2}+2\mathbf{z}_{2,e_2};\\
&{\bf y}_{2,2}^{e_1}= -2\Tilde{\mathbf{h}}_{2,e_1}+3\mathbf{z}_{2,e_1}, \quad &&{\bf y}_{2,2}^{e_2}= -2\Tilde{\mathbf{h}}_{2,e_2}+3\mathbf{z}_{2,e_2};\\
&{\bf y}_{2,3}^{e_1}= -3\Tilde{\mathbf{h}}_{2,e_1}+4\mathbf{z}_{2,e_1}, \quad &&{\bf y}_{2,3}^{e_2}= -3\Tilde{\mathbf{h}}_{2,e_2}+4\mathbf{z}_{2,e_2};\\
&{\bf y}_{3,1}^{e_1}= -\Tilde{\mathbf{h}}_{3,e_1}+2\mathbf{z}_{3,e_1}, \quad &&{\bf y}_{3,1}^{e_2}= -\Tilde{\mathbf{h}}_{3,e_2}+2\mathbf{z}_{3,e_2};\\
&{\bf y}_{3,2}^{e_1}= -2\Tilde{\mathbf{h}}_{3,e_1}+3\mathbf{z}_{3,e_1}, \quad &&{\bf y}_{3,2}^{e_2}= -2\Tilde{\mathbf{h}}_{3,e_2}+3\mathbf{z}_{3,e_2};\\
&{\bf y}_{3,3}^{e_1}= -3\Tilde{\mathbf{h}}_{3,e_1}+4\mathbf{z}_{3,e_1}, \quad &&{\bf y}_{3,3}^{e_2}= -3\Tilde{\mathbf{h}}_{3,e_2}+4\mathbf{z}_{3,e_2},
\vspace{-0.5cm}
\end{alignat*}

Then, each client $v \in [3]$ aggregates the received masked shares from all $3$ clients to obtain
\begin{align*}
\mathbf{y}_{1}=({\bf y}^{e_1}_{1,1} +{\bf y}^{e_1}_{2,1} + {\bf y}^{e_1}_{3,1}, \;  {\bf y}^{e_2}_{1,1} +{\bf y}^{e_2}_{2,1} + {\bf y}^{e_2}_{3,1}), \\
\mathbf{y}_{2}=({\bf y}^{e_1}_{1,2} +{\bf y}^{e_1}_{2,2} + {\bf y}^{e_1}_{3,2}, \;  {\bf y}^{e_2}_{1,2} +{\bf y}^{e_2}_{2,2} + {\bf y}^{e_2}_{3,2}), \\
\mathbf{y}_{3}=({\bf y}^{e_1}_{1,3} +{\bf y}^{e_1}_{2,3} + {\bf y}^{e_1}_{3,3}, \;  {\bf y}^{e_2}_{1,3} +{\bf y}^{e_2}_{2,3} + {\bf y}^{e_2}_{3,3}).
\end{align*}



\noindent\textbf{Private Embedding Aggregation Retrieval.}
We explain how client $1$ privately retrieves its intended global embedding $({\mathbf{h}}_{1,e_1}+{\mathbf{h}}_{3,e_1})/2$ without revealing the entity $e_1$ and similar for others.
Client $1$ samples $2$ random noises $z_1$ and $z_2$ uniformly, and sends coded query $\mathbf{q}_{v}$ to client $v\in[3]$, given by
\begin{IEEEeqnarray*}{c}
\mathbf{q}_{1}\!\!=\!\!(\!-1\!+\!2z_1,2z_2\!), \mathbf{q}_{2}\!\!=\!\!(\!-2\!+\!3z_1,3z_2), \mathbf{q}_{3}\!\!=\!\!(\!-3\!+\!4z_1,4z_2).
\end{IEEEeqnarray*}

Having received the query $\mathbf{q}_{v}$, client $v \in [3]$ computes the inner products $A_{v}=\langle\mathbf{q}_{v},\mathbf{y}_{v}\rangle$ as responses, and sends $\Tilde{A}_v=\texttt{PHE.Enc}({A}_v,pk_1)$ to the server. 

Server samples a random noise $r$ and encrypts locally generated random noises ${\bf s}_1$ and ${\bf s}_2$ to obtain $\Tilde{\bf s}_1=\texttt{PHE.Enc}(\mathbf{s}_1,pk_1),\Tilde{\bf s}_2=\texttt{PHE.Enc}(\mathbf{s}_2,pk_1)$, and then  computes the results $\Tilde{Y}_1,\Tilde{Y}_2,\Tilde{Y}_3$ for client 1. 
\begin{align*}
\Tilde{Y}_{1} = (\Tilde{A}_1 )^r (\Tilde{\mathbf{s}}_1)^3, \;
\Tilde{Y}_{2} = (\Tilde{A}_2)^r  (\Tilde{\mathbf{s}}_2)^3, \;
\Tilde{Y}_{3} = (\Tilde{A}_3)^r (\Tilde{\mathbf{s}}_1)^{-6}(\Tilde{\mathbf{s}}_2)^8.
\end{align*}
Client 1 decrypts the results and gets the responses $Y_1, Y_2, Y_3$.
\begin{align*}
Y_{1} = rA_1 + 3\mathbf{s}_1, \;
Y_{2} = rA_2 + 3\mathbf{s}_2, \;
Y_{3} = rA_3-6\mathbf{s}_1+8\mathbf{s}_2.
\end{align*}

Finally, with $Y_1,Y_2,Y_3$, client $1$ computes $6Y_1-8Y_2+3Y_3=(r(\mathbf{h}_{1,e_1}\!+\mathbf{h}_{3,e_1}),2r),$
and obtains the global embedding of its entity $e_1$ as $\frac{r(\mathbf{h}_{1,e_1}+\mathbf{h}_{3,e_1})}{2r}=\frac{\mathbf{h}_{1,e_1}+\mathbf{h}_{3,e_1}}{2}$.

\begin{Remark} Through this example, the advantage of \scheme over PSI-based protocols can be easily seen. As the intersection of 3 clients' entity sets is $\varnothing$, PSI-based protocol would miss the opportunity to collaborate. In general, \scheme is superior to PSI-based protocols in 1) \scheme builds upon entity union, which captures all possible aggregation opportunities across clients, while PSI only allows within the intersection of all clients' entity sets; 2) \scheme achieves a higher level of entity privacy. While PSI still reveals the intersection of entity sets to clients, local entity set are kept completely private in \scheme.
\end{Remark}
\section{Theoretical Analysis}

We theoretically demonstrate the performance and privacy guarantees provided by the proposed \scheme protocol and analyze its operational complexities.

\subsection{Performance and Privacy Guarantees}

\begin{theorem}\label{privacy theorem}
The proposed \scheme protocol for general federated representation learning tasks can leverage all potential aggregation opportunities among all the clients and is $T$-private for any $T< {N}/{2}$, i.e., it simultaneously achieves entity privacy and embedding privacy against 1) the curious server in a computational sense, and 2) any subset of up to $T$ colluding clients in a statistical sense (except for the union of all entity sets).
\end{theorem}

\begin{proof}
We know from \eqref{rnoise} that each client, for each of its local entities, can accurately obtain the average embedding across all clients who have this entity locally. Thus our \scheme protocol leverages all potential aggregation opportunities among all clients. Moreover, in \scheme, all messages received by the server are masked by pseudo-random noises using one-time-pad encryption or are encrypted using $\texttt{PHE}$, and hence the server learns nothing about the entities and embeddings, due to the computational security of one-time-pad encryption and $\texttt{PHE}$.
Moreover, in the phases of private embedding sharing and private embedding aggregation retrieval, all shared data between clients are masked by $T$ random noises, which 
admit entity privacy and embedding privacy against any subset of up to $T < {N}/{2}$ colluding clients. See Appendix for details.
\end{proof}

\subsection{Complexity Analysis}\label{complexity analysis}

As the private entity union operation 
is performed only once before the training rounds, we expect a negligible contribution of its complexity to the overhead of the \scheme protocol.

We next analyze the complexities of secure embedding sharing and private embedding aggregation retrieval, which are operations carried out in each training round of the \scheme protocol. We note that the queries in \eqref{query:1} and the \texttt{PHE} noise terms $\texttt{PHE.Enc}(\psi_{n,e}(\alpha_{v}),pk_n)$ in \eqref{answers:add} are constructed \emph{independently} of the entity embeddings, and thus can be computed and stored \emph{offline} before each round starts. These offline storage and computation costs are analyzed as follows.

\noindent\textbf{Offline Storage Cost.} The offline storage contains the queries in \eqref{query:1} and the encrypted noise terms $\texttt{PHE.Enc}({\psi}_{n,e}(\alpha_{v}))$ in \eqref{answers:add} that are both used in private embedding aggregation retrieval. Each client $n$ generates a query vector $\mathbf{q}_{n,v,e}$ of length $M$ sent to each client $v\in[N]$ for each entity $e\in\mathcal{E}_n$. Thus, the storage cost at client $n$ is $O(MN|\mathcal{E}_n|)$ over the finite field $\mathbb{F}$, where $|\mathcal{ E}_n|$ denotes the cardinality of $\mathcal{ E}_n$.
For each client $n$, the server can independently generate the encrypted noise  $\texttt{PHE.Enc}(\psi_{n,e}(\alpha_v))$ of dimension $\frac{d+1}{K}$ for each $e\in\mathcal{E}_n$ and  $n,v\in[N]$. Hence,
the total offline storage cost at server is $O(\frac{dN\sum_{n=1}^{N}|\mathcal{E}_n|}{K})$ over the ciphertext space $\mathbb{Q}$, where $\mathbb{Q}$ is the maximum \texttt{PHE} ciphertext space over the $N$ clients.


\noindent\textbf{Offline Computation Cost.} The offline computation includes generating queries at each client and encrypted noise terms at the server. The queries $\{\rho_{n,e}^{m}(\alpha_v)\}_{v\in[N]}$  \eqref{query:1} at client $n$ are generated by evaluating the polynomial $\rho_{n,e}^{m}(x)$ of degree $K+T-1$ at $N$ points, for each $m\in[M]$ and $e\in\mathcal{E}_n$. This can be done with complexity $O(MN(\log N)^2|\mathcal{E}_n|)$ \cite{Von}.
For the encrypted noise terms \eqref{answers:add}, the server first generates the noise terms $\{\psi_{n,e}(\alpha_{v})\}_{v\in[N]}$ of dimension $\frac{d+1}{K}$  by evaluating the polynomial $\psi_{n,e}(x)$ of degree $2(K+T-1)<N$ at $N$ points and then encrypts these evaluations, for all $n\in[N]$ and $e\in\mathcal{E}_n$. The former for polynomial evaluations yields a complexity of $O(\frac{dN(\log N)^2\sum_{n=1}^{N}|\mathcal{E}_n|}{K})$. We know that both the encryption and decryption of $\texttt{PHE}$ can be achieved within a complexity of $(\log Q)^3$ where $Q=|\mathbb{Q}|$. Thus the latter for encrypting these evaluations incurs a complexity of $O(\frac{dN(\log Q)^3\sum_{n=1}^{N}|\mathcal{E}_n|}{K})$. 
So, the total offline computational complexity at the server is
$O(\frac{dN(\log Q)^3\sum_{n=1}^{N}|\mathcal{E}_n|}{K})$.


\noindent\textbf{Online Communication Cost.} The online communication overhead at each client $n$ consists of three parts: 1) sending the one-time-pad encrypted share $\Tilde{\mathbf{y}}_{n,v}$ \eqref{stored data} of dimension $\frac{M(d+1)}{K}$ to each client $v\in[N]$; 2) sending the one-time-pad encrypted query $\Tilde{\mathbf{q}}_{n,v,e}$ \eqref{query:1} of dimension $M$ to each client $v\in[N]$, for each entity $e\in\mathcal{E}_n$; 3) responding the $\texttt{PHE}$ answer $\Tilde{A}_{n,v,e}$  
of dimension $\frac{d+1}{K}$ to client $v$ for each $v\in[N]$ and $e\in\mathcal{E}_v$.
The incurred communication overhead of client $n$ for these three parts are $O(\frac{dMN}{K})$ and $O(MN|\mathcal{E}_n|)$ in $\mathbb{F}$, and $O(\frac{d\sum_{v=1}^{N}|\mathcal{E}_v|}{K})$ in $\mathbb{Q}$, respectively. Thus, the total online communication overhead at client $n$ is $O(\frac{dMN}{K}\log|\mathbb{F}|+MN|\mathcal{E}_n|\log|\mathbb{F}|+\frac{d\sum_{n=1}^{N}|\mathcal{E}_n|}{K}\log Q$).

\noindent\textbf{Online Computation Cost.} The online computational overhead at client $n$ contains four parts: 
1) generating the encoded data $\{\varphi_{n,e_m}(\alpha_v)\}_{v\in[N]}$ sent to the $N$ clients for all $m\in[M]$. This can be viewed as evaluating the polynomial $\varphi_{v,e_m}(x)$ in \eqref{encoding:data} of degree $K+T-1<N$ at $N$ points for $\frac{d+1}{K}$ times for each $m\in[M]$, and thus achieves a complexity $O(\frac{dMN(\log N)^2)}{K})$ \cite{Von}; 
2) generating the answer $A_{n,v,e}$ 
to client $v$ by computing a linear combination of two vectors of dimension $M$ for $\frac{d+1}{K}$ times for each $v\in[N]$ and each $e\in\mathcal{E}_v$, which incurs a complexity of $O(\frac{dM\sum_{v=1}^{N}|\mathcal{E}_v|}{K})$;
3) encrypting the answer ${A}_{n,v,e}$ of dimension $\frac{d+1}{K}$ to client $v$ for each $v\in[N]$ and $e\in\mathcal{E}_v$ using \texttt{PHE}, and decrypting $N$ \texttt{PHE} responses of each dimension $\frac{d+1}{K}$ for each $e\in\mathcal{E}_n$, which  incur the complexities $O(\frac{d(\log Q)^3\sum_{v=1}^{N}|\mathcal{E}_v|}{K})$ and $O(\frac{dN(\log Q)^3|\mathcal{E}_n|}{K})$, respectively;
and 4) decoding the embedding aggregation of entity $e$  with dimension $\frac{d+1}{K}$ by first interpolating
$Y_{n,e}(x)$ of degree $2(K+T-1)<N$, and then evaluating it at $K$ points, which yields a computational complexity of $O(\frac{dN(\log N)^2|\mathcal{E}_n|}{K})$ for all entities in $\mathcal{E}_n$.  
The online computation \eqref{answers:add} at server mainly consists of raising the response $\Tilde{A}_{n,v,e}$ of dimension $\frac{d+1}{K}$ to the power of $r_{n,e}<Q$ for each $n,v\in[N]$ and $e\in\mathcal{E}_v$, which incurs a complexity of $O(\frac{dNQ\sum_{n=1}^{N}|\mathcal{E}_n|}{K})$.

\section{Conclusion}
We proposed a novel secure embedding aggregation framework \scheme for federated representation learning, which leverages all potential aggregation opportunities among all the clients while ensuring entity privacy and embedding privacy simultaneously. 
We theoretically demonstrated that \scheme achieves provable privacy against a curious server and a threshold number of colluding clients, and analyzed its complexities.


\section*{acknowledgement}
The work of Jiaxiang Tang and Songze Li is in part supported by the National Nature Science Foundation of China (NSFC) Grant 62106057, Guangzhou Municipal Science and Technology Guangzhou-HKUST(GZ) Joint Project 2023A03J0151 and Project 2023A03J0011, Foshan HKUST Projects FSUST20-FYTRI04B, and Guangdong Provincial Key Lab of Integrated Communication, Sensing and Computation for Ubiquitous Internet of Things. The work of Jinbao Zhu is supported by the China Postdoctoral Science Foundation Grant 2022M720891.

\bibliographystyle{ieeetr}
\bibliography{reference.bib}

\clearpage

\appendix[Proof of Theorem 1]\label{proof}
After securely obtaining the set of all entities on all clients using  private entity union, in each training round, the clients aggregate their embeddings for each of the entities, using private embedding sharing and private embedding aggregation retrieval, such that no entity or embedding information about individual clients is leaked to the server or any subset of up to $T < {N}/{2}$ colluding clients. 

The entity privacy in the private entity union is provided by \cite[Lemma 1]{PSU} in a statistical sense.
Next, we move on to prove that entity privacy and embedding privacy are preserved, in the phases of private embedding sharing and private embedding aggregation retrieval. Over the course of these two phases, all shares of embeddings received by the server are masked by pseudo-random noises using one-time-pad encryption or are encrypted via $\texttt{PHE}$, and hence the server learns nothing about the embeddings and entities. 
Finally, we show that, in the information-theoretic sense, the phases of private embedding sharing and private embedding aggregation retrieval admit entity privacy and embedding privacy against any subset of up to $T <{N}/{2}$ colluding clients, other than the global embeddings of the entities owned by these $T$ colluding clients.
This is made precise in the following lemma. Overall, based on the statistical security of entity privacy and information-theoretic security of embedding privacy, we achieve the statistical security against colluding clients.

\begin{lemma}
For any subset of clients $\mathcal{T}\subseteq[N]$ of size $T < {N}/{2}$, we have
\begin{IEEEeqnarray}{l}
I\bigg(\{\mathcal{E}_n\}_{n\in[N]\backslash\mathcal{T}},\{\mathbf{h}_{n,e}:e\in\mathcal{E}_n\}_{n\in[N]\backslash\mathcal{T}};\mathbf{R}_{\mathcal{T}}\notag\\
\quad\Big|\big\{\frac{\sum_{v\in[N]}\mathbbm{1}(e\in\mathcal{E}_{v})\cdot\mathbf{h}_{v,e}}{\sum_{v\in[N]}\mathbbm{1}(e\in\mathcal{E}_{v})},\mathcal{E}_n,\mathbf{h}_{n,e}:e\in\mathcal{E}_n\big\}_{n\in\mathcal{T}}\bigg)=0, \notag
\end{IEEEeqnarray}
where $\mathbf{R}_{\mathcal{T}}$ denotes all messages received by these $T$ colluding clients in the phases of private embedding sharing and private embedding aggregation retrieval.
\end{lemma}

\begin{proof}
In our \scheme protocol, $\mathbf{R}_{\mathcal{T}}$ includes the shared data $\{\mathbf{y}_{v,n}\}_{v\in[N],n\in\mathcal{T}}$  \eqref{stored data} received by clients $\mathcal{T}$ from clients $[N]$ during the phase of private embedding sharing, and the queries  $\{\mathbf{q}_{v,n,e}\}_{v\in[N],n\in\mathcal{T},e\in\mathcal{E}_v}$ in \eqref{query:1} and the responses $\{Y_{v,n,e}:e\in\mathcal{E}_n\}_{v\in[N],n\in\mathcal{T}}$ in \eqref{answers:add} received by clients $\mathcal{T}$ from clients $[N]$, in the phase of private embedding aggregation retrieval.
Thus, we have
\begin{IEEEeqnarray}{rCl}
0&\leq&I\big(\{\mathcal{E}_n\}_{n\in[N]\backslash\mathcal{T}},\{\mathbf{h}_{n,e}:e\in\mathcal{E}_n\}_{n\in[N]\backslash\mathcal{T}};\notag\\
&&\quad\mathbf{R}_{\mathcal{T}}\Big|\big\{\frac{\sum_{v\in[N]}\mathbbm{1}(e\in\mathcal{E}_{v})\cdot\mathbf{h}_{v,e}}{\sum_{v\in[N]}\mathbbm{1}(e\in\mathcal{E}_{v})},\mathcal{E}_n,\mathbf{h}_{n,e}:e\in\mathcal{E}_n\big\}_{n\in\mathcal{T}}) \notag\\
&=&I\big(\{\mathcal{E}_n\}_{n\in[N]\backslash\mathcal{T}},\{\mathbf{h}_{n,e}\!:\!e\!\in\!\mathcal{E}_n\}_{n\in[N]\backslash\mathcal{T}};\{\mathbf{y}_{v,n}\}_{v\in[N],n\in\mathcal{T}},\notag\\
&&\{\mathbf{q}_{v,n,e}\}_{v\in[N],n\in\mathcal{T},e\in\mathcal{E}_v},\{Y_{v,n,e}:e\in\mathcal{E}_n\}_{v\in[N],n\in\mathcal{T}} \notag\\
&&\quad\quad\Big|\big\{\frac{\sum_{v\in[N]}\mathbbm{1}(e\in\mathcal{E}_{v})\cdot\mathbf{h}_{v,e}}{\sum_{v\in[N]}\mathbbm{1}(e\in\mathcal{E}_{v})},\mathcal{E}_n,\mathbf{h}_{n,e}:e\in\mathcal{E}_n\big\}_{n\in\mathcal{T}}\big) \notag\\
&=&I\big(\{\mathcal{E}_n\}_{n\in[N]\backslash\mathcal{T}},\{\mathbf{h}_{n,e}:e\in\mathcal{E}_n\}_{n\in[N]\backslash\mathcal{T}};\{\mathbf{y}_{v,n}\}_{v\in[N],n\in\mathcal{T}}\notag\\
&&\quad\quad\Big|\big\{\frac{\sum_{v\in[N]}\mathbbm{1}(e\in\mathcal{E}_{v})\cdot\mathbf{h}_{v,e}}{\sum_{v\in[N]}\mathbbm{1}(e\in\mathcal{E}_{v})},\mathcal{E}_n,\mathbf{h}_{n,e}:e\in\mathcal{E}_n\big\}_{n\in\mathcal{T}}\big) \notag\\
&&+I\big(\{\mathcal{E}_n\}_{n\in[N]\backslash\mathcal{T}},\{\mathbf{h}_{n,e}:e\in\mathcal{E}_n\}_{n\in[N]\backslash\mathcal{T}};\notag\\
&&\quad\{\mathbf{q}_{v,n,e}\}_{v\in[N],n\in\mathcal{T},e\in\mathcal{E}_v},
\{Y_{v,n,e}:e\in\mathcal{E}_n\}_{v\in[N],n\in\mathcal{T}}\notag\\
&&\quad\quad\Big|\{\mathbf{y}_{v,n}\}_{v\in[N],n\in\mathcal{T}},\notag\\
&&\quad\quad\big\{\frac{\sum_{v\in[N]}\mathbbm{1}(e\in\mathcal{E}_{v})\cdot\mathbf{h}_{v,e}}{\sum_{v\in[N]}\mathbbm{1}(e\in\mathcal{E}_{v})},\mathcal{E}_n,\mathbf{h}_{n,e}:e\in\mathcal{E}_n\big\}_{n\in\mathcal{T}}\big) \notag\\
&\overset{(a)}{=}&I\big(\{\mathcal{E}_n\}_{n\in[N]\backslash\mathcal{T}},\{\mathbf{h}_{n,e}:e\in\mathcal{E}_n\}_{n\in[N]\backslash\mathcal{T}};\notag\\
&&\quad\{\mathbf{q}_{v,n,e}\}_{v\in[N],n\in\mathcal{T},e\in\mathcal{E}_v}\Big|\{\mathbf{y}_{v,n}\}_{v\in[N],n\in\mathcal{T}},\notag\\
&&\quad\quad\big\{\frac{\sum_{v\in[N]}\mathbbm{1}(e\in\mathcal{E}_{v})\cdot\mathbf{h}_{v,e}}{\sum_{v\in[N]}\mathbbm{1}(e\in\mathcal{E}_{v})},\mathcal{E}_n,\mathbf{h}_{n,e}:e\in\mathcal{E}_n\big\}_{n\in\mathcal{T}}\big)\notag\\
&&+I\big(\{\mathcal{E}_n\}_{n\in[N]\backslash\mathcal{T}},\{\mathbf{h}_{n,e}:e\in\mathcal{E}_n\}_{n\in[N]\backslash\mathcal{T}};\notag\\
&&\quad\{Y_{v,n,e}:e\in\mathcal{E}_n\}_{v\in[N],n\in\mathcal{T}}\notag\\
&&\quad\quad\Big|\{\mathbf{q}_{v,n,e}\}_{v\in[N],n\in\mathcal{T},e\in\mathcal{E}_v},\{\mathbf{y}_{v,n}\}_{v\in[N],n\in\mathcal{T}},\notag\\
&&\quad\quad\big\{\frac{\sum_{v\in[N]}\mathbbm{1}(e\in\mathcal{E}_{v})\cdot\mathbf{h}_{v,e}}{\sum_{v\in[N]}\mathbbm{1}(e\in\mathcal{E}_{v})},\mathcal{E}_n,\mathbf{h}_{n,e}:e\in\mathcal{E}_n\big\}_{n\in\mathcal{T}}\big)\notag\\
&\overset{(b)}{=}&I\big(\{\mathcal{E}_n\}_{n\in[N]\backslash\mathcal{T}},\{\mathbf{h}_{n,e}:e\in\mathcal{E}_n\}_{n\in[N]\backslash\mathcal{T}};\notag\\
&&\quad\{Y_{v,n,e}:e\in\mathcal{E}_n\}_{v\in[N],n\in\mathcal{T}}\notag\\
&&\quad\quad\Big|\{\mathbf{q}_{v,n,e}\}_{v\in[N],n\in\mathcal{T},e\in\mathcal{E}_v},\{\mathbf{y}_{v,n}\}_{v\in[N],n\in\mathcal{T}},\notag\\
&&\quad\quad\big\{\frac{\sum_{v\in[N]}\mathbbm{1}(e\in\mathcal{E}_{v})\cdot\mathbf{h}_{v,e}}{\sum_{v\in[N]}\mathbbm{1}(e\in\mathcal{E}_{v})},\mathcal{E}_n,\mathbf{h}_{n,e}:e\in\mathcal{E}_n\big\}_{n\in\mathcal{T}}\big)\notag\\
&\overset{(c)}{=}&I\big(\{\mathcal{E}_n\}_{n\in[N]\backslash\mathcal{T}},\{\mathbf{h}_{n,e}:e\in\mathcal{E}_n\}_{n\in[N]\backslash\mathcal{T}};\notag\\
&&\big\{Y_{n,e}(x)\!:\!x\in\{\beta_k\}_{k\in[K]}\!\cup\!\{\alpha_k\}_{k\in[K+2T-1]}\!:\!e\in\mathcal{E}_n\big\}_{n\in\mathcal{T}}\notag\\
&&\quad\quad\Big|\{\mathbf{q}_{v,n,e}\}_{v\in[N],n\in\mathcal{T},e\in\mathcal{E}_v},\{\mathbf{y}_{v,n}\}_{v\in[N],n\in\mathcal{T}},\notag\\
&&\quad\quad\big\{\frac{\sum_{v\in[N]}\mathbbm{1}(e\in\mathcal{E}_{v})\cdot\mathbf{h}_{v,e}}{\sum_{v\in[N]}\mathbbm{1}(e\in\mathcal{E}_{v})},\mathcal{E}_n,\mathbf{h}_{n,e}:e\in\mathcal{E}_n\big\}_{n\in\mathcal{T}}\big)\notag\\
&\overset{(d)}{=}&I\big(\{\mathcal{E}_n\}_{n\in[N]\backslash\mathcal{T}},\{\mathbf{h}_{n,e}:e\in\mathcal{E}_n\}_{n\in[N]\backslash\mathcal{T}};\notag\\
&&\big\{\{\Lambda_{n,e}(\alpha_k)+\mathbf{z}_{n,e}^{k}\}_{k\in[K+2T-1]},\{r_{n,e}\sum\limits_{v\in[N]}\Tilde{\mathbf{h}}_{v,e}^{k}\}_{k\in[K]}:\notag\\
&&\quad e\in\mathcal{E}_n\big\}_{n\in\mathcal{T}}\Big|\{\mathbf{q}_{v,n,e}\}_{v\in[N],n\in\mathcal{T},e\in\mathcal{E}_v},\{\mathbf{y}_{v,n}\}_{v\in[N],n\in\mathcal{T}},\notag\\
&&\quad\quad\big\{\frac{\sum_{v\in[N]}\mathbbm{1}(e\in\mathcal{E}_{v})\cdot\mathbf{h}_{v,e}}{\sum_{v\in[N]}\mathbbm{1}(e\in\mathcal{E}_{v})},\mathcal{E}_n,\mathbf{h}_{n,e}:e\in\mathcal{E}_n\big\}_{n\in\mathcal{T}}\big)\notag\\
&\overset{(e)}{=}&I\big(\{\mathcal{E}_n\}_{n\in[N]\backslash\mathcal{T}},\{\mathbf{h}_{n,e}:e\in\mathcal{E}_n\}_{n\in[N]\backslash\mathcal{T}};\notag\\
&&\quad\big\{\{\Lambda_{n,e}(\alpha_k)+\mathbf{z}_{n,e}^{k}\}_{k\in[K+2T-1]}: e\in\mathcal{E}_n\big\}_{n\in\mathcal{T}},\notag\\
&&\big\{r_{n,e}\!{\sum_{v\in[N]}\!\!\mathbbm{1}(e\!\in\!\mathcal{E}_{v})\!\cdot\!\mathbf{h}_{v,e}},r_{n,e}\!{\sum_{v\in[N]}\!\!\mathbbm{1}(e\!\in\!\mathcal{E}_{v})}\!:\!e\in\mathcal{E}_n\big\}_{n\in\mathcal{T}}\notag\\
&&\Big|\{\mathbf{q}_{v,n,e}\}_{v\in[N],n\in\mathcal{T},e\in\mathcal{E}_v},\{\mathbf{y}_{v,n}\}_{v\in[N],n\in\mathcal{T}},\notag\\
&&\quad\quad\big\{\frac{\sum_{v\in[N]}\mathbbm{1}(e\in\mathcal{E}_{v})\cdot\mathbf{h}_{v,e}}{\sum_{v\in[N]}\mathbbm{1}(e\in\mathcal{E}_{v})},\mathcal{E}_n,\mathbf{h}_{n,e}:e\in\mathcal{E}_n\big\}_{n\in\mathcal{T}}\big)\notag\\
&{=}&I\big(\{\mathcal{E}_n\}_{n\in[N]\backslash\mathcal{T}},\{\mathbf{h}_{n,e}:e\in\mathcal{E}_n\}_{n\in[N]\backslash\mathcal{T}};\notag\\
&&\quad\big\{\{\Lambda_{n,e}(\alpha_k)+\mathbf{z}_{n,e}^{k}\}_{k\in[K+2T-1]}: e\in\mathcal{E}_n\big\}_{n\in\mathcal{T}},\notag\\
&&\big\{\frac{\sum_{v\in[N]}\!\mathbbm{1}(e\!\in\!\mathcal{E}_{v})\!\cdot\!\mathbf{h}_{v,e}}{\sum_{v\in[N]}\mathbbm{1}(e\!\in\!\mathcal{E}_{v})},r_{n,e}\!{\sum_{v\in[N]}\!\mathbbm{1}(e\!\in\!\mathcal{E}_{v})}:e\!\in\!\mathcal{E}_n\big\}_{n\in\mathcal{T}}\notag\\
&&\Big|\{\mathbf{q}_{v,n,e}\}_{v\in[N],n\in\mathcal{T},e\in\mathcal{E}_v},\{\mathbf{y}_{v,n}\}_{v\in[N],n\in\mathcal{T}},\notag\\
&&\quad\big\{\frac{\sum_{v\in[N]}\mathbbm{1}(e\in\mathcal{E}_{v})\cdot\mathbf{h}_{v,e}}{\sum_{v\in[N]}\mathbbm{1}(e\in\mathcal{E}_{v})},\mathcal{E}_n,\mathbf{h}_{n,e}:e\in\mathcal{E}_n\big\}_{n\in\mathcal{T}}\big)\notag\\
&{=}&I\big(\{\mathcal{E}_n\}_{n\in[N]\backslash\mathcal{T}},\{\mathbf{h}_{n,e}:e\in\mathcal{E}_n\}_{n\in[N]\backslash\mathcal{T}};\notag\\
&&\quad\big\{\{\Lambda_{n,e}(\alpha_k)+\mathbf{z}_{n,e}^{k}\}_{k\in[K+2T-1]}: e\in\mathcal{E}_n\big\}_{n\in\mathcal{T}},\notag\\
&&\big\{r_{n,e}{\sum_{v\in[N]}\mathbbm{1}(e\in\mathcal{E}_{v})}:e\in\mathcal{E}_n\big\}_{n\in\mathcal{T}}\notag\\
&&\Big|\{\mathbf{q}_{v,n,e}\}_{v\in[N],n\in\mathcal{T},e\in\mathcal{E}_v},\{\mathbf{y}_{v,n}\}_{v\in[N],n\in\mathcal{T}},\notag\\
&&\;\big\{\frac{\sum_{v\in[N]}\mathbbm{1}(e\in\mathcal{E}_{v})\cdot\mathbf{h}_{v,e}}{\sum_{v\in[N]}\mathbbm{1}(e\in\mathcal{E}_{v})},\mathcal{E}_n,\mathbf{h}_{n,e}:e\in\mathcal{E}_n\big\}_{n\in\mathcal{T}}\big)\label{terms}\\
&\overset{(f)}{=}&0.\notag
\end{IEEEeqnarray}
Here $(a)$ is because $\{\mathbf{y}_{v,n}\}_{v\in[N],n\in\mathcal{T}}=\{\varphi_{v,e_m}(\alpha_n):m\in[M],v\in[N],n\in\mathcal{T}\}$ by \eqref{stored data} and the data $\{\varphi_{v,e_m}(\alpha_n)\}_{n\in\mathcal{T}}$ received by the clients $\mathcal{T}$ are protected by $T$ independent and uniform random noises $\mathbf{z}^{K+1}_{v,e_m},\ldots,\mathbf{z}^{K+T}_{v,e_m}$ for all $m\in[M],v\in[N]$ by \eqref{encoding:data}, such that
$\{\mathbf{y}_{v,n}\}_{v\in[N],n\in\mathcal{T}}$ are independent of $\{\mathcal{E}_n\}_{n\in[N]\backslash\mathcal{T}}$, $\{\mathbf{h}_{n,e}:e\in\mathcal{E}_n\}_{n\in[N]\backslash\mathcal{T}}$ and $\big\{\frac{\sum_{v\in[N]}\mathbbm{1}(e\in\mathcal{E}_{v})\cdot\mathbf{h}_{v,e}}{\sum_{v\in[N]}\mathbbm{1}(e\in\mathcal{E}_{v})},\mathcal{E}_n,\mathbf{h}_{n,e}:e\in\mathcal{E}_n\big\}_{n\in\mathcal{T}}$. Thus 
\begin{IEEEeqnarray}{l}
I\big(\{\mathcal{E}_n\}_{n\in[N]\backslash\mathcal{T}},\{\mathbf{h}_{n,e}:e\in\mathcal{E}_n\}_{n\in[N]\backslash\mathcal{T}};\{\mathbf{y}_{v,n}\}_{v\in[N],n\in\mathcal{T}}\notag\\
\quad\quad\Big|\big\{\frac{\sum_{v\in[N]}\mathbbm{1}(e\in\mathcal{E}_{v})\cdot\mathbf{h}_{v,e}}{\sum_{v\in[N]}\mathbbm{1}(e\in\mathcal{E}_{v})},\mathcal{E}_n,\mathbf{h}_{n,e}:e\in\mathcal{E}_n\big\}_{n\in\mathcal{T}}\big)=0.\notag
\end{IEEEeqnarray}
The step $(b)$ is similar to $(a)$ because $\{\mathbf{q}_{v,n,e}\}_{v\in[N],n\in\mathcal{T},e\in\mathcal{E}_v}=\{\rho_{v,e}^{m}(\alpha_n)\}_{m\in[M],v\in[N],n\in\mathcal{T},e\in\mathcal{E}_v}$ by \eqref{query:1} and
the queries $\{\rho_{v,e}^{m}(\alpha_n)\}_{n\in\mathcal{T}}$ received by the clients $\mathcal{T}$ are protected by $T$ independent and uniform random noises $z_{v,e}^{m,K+1},\ldots,z_{v,e}^{m,K+T}$ for all $m\in[M],v\in[N]$ and $e\in\mathcal{E}_v$ by \eqref{symmetric:22}, such that
\begin{IEEEeqnarray}{l}
   I\big(\{\mathcal{E}_n\}_{n\in[N]\backslash\mathcal{T}},\{\mathbf{h}_{n,e}:e\in\mathcal{E}_n\}_{n\in[N]\backslash\mathcal{T}};\notag\\
\quad\{\mathbf{q}_{v,n,e}\}_{v\in[N],n\in\mathcal{T},e\in\mathcal{E}_v}\Big|\{\mathbf{y}_{v,n}\}_{v\in[N],n\in\mathcal{T}},\notag\\
\quad\quad\big\{\frac{\sum_{v\in[N]}\mathbbm{1}(e\in\mathcal{E}_{v})\cdot\mathbf{h}_{v,e}}{\sum_{v\in[N]}\mathbbm{1}(e\in\mathcal{E}_{v})},\mathcal{E}_n,\mathbf{h}_{n,e}:e\in\mathcal{E}_n\big\}_{n\in\mathcal{T}}\big)=0. \notag
\end{IEEEeqnarray}
The step $(c)$ holds because the answer $Y_{v,n,e}$ is equivalent to evaluating $Y_{n,e}(x)$ at $x=\alpha_{v}$ for any $v\in[N]$ and $Y_{n,e}(x)$ is a polynomial of degree $2(K+T-1)$ by \eqref{answer:polynomial}, such that $\{Y_{1,n,e}, \ldots, Y_{N,n,e}\}$ and $\{Y_{n,e}(x):x\in\{\beta_k\}_{k\in[K]}\cup\{\alpha_k\}_{k\in[K+2T-1]}\}$ are determined of each other by Lagrange interpolation rules for any $e\in\mathcal{E}_n$ and $n\in\mathcal{T}$;
$(d)$ follows by \eqref{symmetric:1345}-\eqref{symmetric:1345} and \eqref{answer:polynomial}-\eqref{evaluating:1} in which $\Lambda_{n,e}(x)\triangleq r_{n,e}\sum_{m=1}^{M}\rho_{n,e}^{m}(x)\cdot \sum_{v'\in[N]}\varphi_{v',e_m}(x)$;
$(e)$ is due to \eqref{rnoise};
$(f)$ follows from the fact that $\{\{\mathbf{z}_{n,e}^{k}\}_{k\in[K+2T-1]},r_{n,e}:e\in\mathcal{E}_n\}_{n\in\mathcal{T}}$ are i.i.d. uniformly over $\mathbb{F}$ and are generated independently of all other variables in \eqref{terms}.

This completes the proof of the lemma.
\end{proof}
\end{document}